%% file: main.tex
\documentclass{article}
\usepackage{spconf,amsmath,graphicx,hyperref}
\usepackage{cite}
\usepackage{amsmath,amssymb,amsfonts}
\usepackage{algorithm}
\usepackage{multirow}
\usepackage{algpseudocode}
\usepackage{graphicx}
\usepackage{textcomp}
\usepackage{xcolor}
\usepackage{etoolbox}
\usepackage{amsthm}
\usepackage{booktabs, multirow}
\usepackage{setspace}
\input{preamble}

\newtheorem{proposition}{Proposition}
\ninept

\title{GLUE: Gradient-free Learning to Unify Experts}
\name{Jong-Ik Park*, Shreyas Chaudhari*, Srinivasa Pranav*, Carlee Joe-Wong, and Jos\'e M. F. Moura\thanks{*Authors contributed equally.\\
Authors partially supported by NSF Grants CNS-2106891, CNS-2409138, and CCF-2327905. S. Chaudhari and S. Pranav partially supported by NSF Graduate Research Fellowships (GRFP; Grants DGE1745016, DGE2140739). S. Pranav partially supported by an ARCS Fellowship.}}
\address{Electrical and Computer Engineering, Carnegie Mellon University}
\begin{document}
\maketitle
%
% \setstretch{1.25}
% \large
\begin{abstract}
In many deployed systems (multilingual ASR, cross-hospital imaging, region-specific perception), multiple pretrained specialist models coexist. Yet, new target domains often require \emph{domain expansion}: a generalized model that performs well beyond any single specialist's domain.
Given a new target domain, existing methods obtain a \emph{single} strong initialization prior for the model parameters by blending expert models to \emph{initialize} a target model. However, heuristic blending---using mixing coefficients based on data size or proxy metrics---often yields \emph{lower target-domain test accuracy}, and learning these coefficients on the target domain's loss function typically requires computationally-expensive full backpropagation through a neural network.
We propose \textbf{GLUE}, \emph{Gradient-free Learning to Unify Experts}, which initializes the target model as a convex combination of fixed experts and learns the mixture coefficients of this combination via gradient-free two-point SPSA (simultaneous perturbation stochastic approximation) updates, requiring only two forward passes per step.
Across experiments on three datasets and three network architectures, GLUE produces model parameter priors that can be fine-tuned to outperform baselines. GLUE improves test accuracy by up to \textbf{8.5\%} over data-size weighting and by up to \textbf{9.1\%} over proxy-metric selection. GLUE either outperforms backpropagation-based full-gradient mixing or matches its performance within \textbf{1.4\%}.
\end{abstract}


\input{sections/introduction}
\input{sections/relatedworks}
\input{sections/methodology}
\input{sections/theoreticalanalysis}

\input{sections/experimentalevaluation}
\input{sections/discussion}
\input{sections/conclusion}

\bibliographystyle{IEEEbib}
\bibliography{mybib}

\end{document}

%% file: preamble.tex
\def\btheta{\boldsymbol{\theta}}
\def\balpha{\boldsymbol{\alpha}}
\def\bbeta{\boldsymbol{\beta}}

\def\g{\mathbf{g}} 
 
\def\x{\mathbf{x}} 
\def\y{\mathbf{y}} 

\def\u{\mathbf{u}}

\def\I{\mathbf{I}}

\def\bTheta{\boldsymbol{\Theta}}

\def\E{\mathbb{E}}
\def\R{\mathbb{R}}

\def\1{\mathbbm{1}}

%% file: sections/introduction.tex
\section{Introduction}
Pretrained parameters, whether from specialist ``expert'' models or large foundation models, enable practitioners to quickly start learning new tasks by using the pretrained parameters to \emph{initialize} a target network and finetune it on the new task data~\cite{noroozi2018boosting,you2021logme,park2025fedbaf}.
This practice is common in multilingual or multi-accent automatic speech recognition (ASR)~\cite{turan2020achieving}, radiology models deployed across scanners and hospitals~\cite{akbarian2023evaluating}, perception stacks that face region/weather variation~\cite{shu2022hub, pal2024transfer}, recommendation systems partitioned by locale or user segments~\cite{fu2024exploring}, and foundation model adapters specialized for code, legal, or biomedical text~\cite{yuan2023power, zhou2024comprehensive, park2025fedbaf}. 
In these settings, initializing a target network with expert parameters often outperforms knowledge distillation \cite{hinton2015distilling, chaudhari2023knowledge}, which requires extra hyperparameters and can introduce optimization noise (e.g., temperature tuning, teacher-student mismatch). In contrast, directly initializing a target network with experts’ parameters simplifies downstream fine-tuning while preserving pretrained feature hierarchy and inductive biases~\cite{zagoruyko2016paying, micaelli2019zero, mirzadeh2020improved, moslemi2024survey}.

In general, adopting a \emph{single} expert's parameters is often suboptimal when building a general-purpose model, because real-world data are fragmented across related but distinct distributions~\cite{cohenask, shi2022robust, xiang2020learning}.
Different experts capture complementary structures~\cite{yuan2021reinforced, wu2021one, boiarov2022simultaneous}, motivating a parameter-space fusion of these experts as an initialization prior. Previous works~\cite{mcmahan2017communication, NonIID_P2P_Asilomar, 5GIoT_P2P_Asilomar, NTK_P2P_ICASSP, wortsman2022model, shi2025fedawa} form convex combinations of the $P$-dimensional parameters of $K$ expert models
$\{\btheta_i\}_{i=1}^K$
\begin{align}
    \btheta(\balpha)=\sum_{i=1}^K \alpha_i \btheta_i,\;\btheta_i \in \R^P
\label{eq:blending}
\end{align}
where the mixture coefficients ${\balpha}=(\alpha_1,\ldots,\alpha_K)$ define a single deployable model tailored to the target domain.

Common heuristics for finding $\balpha$ include (i) weighting experts by the size of their pretraining data~\cite{mcmahan2017communication}, (ii) weighting by an external metric such as held-out accuracy on a proxy dataset\cite{wortsman2022model,matena2022merging, park2025fedbaf}, or (iii) \emph{learning} the weights by optimizing the target loss\cite{ilharcoediting,yadav2023ties,yangadamerging}.
Among these, learning the weights $\balpha$ generally leads to the strongest single initialization (prior) for fine-tuning, because it optimizes the true target loss and can capture interactions among experts that global heuristics miss~\cite{mu2025comprehensive,wortsman2022model}.

\emph{The practical challenge is computational cost}. Even when the number of mixture weights $K$ is small, standard gradient descent on the mixture coefficients still requires computing the loss gradient $\nabla_{\btheta(\balpha)} \mathcal{L}$ for the full parameter vector $\btheta(\balpha)$ of dimension $P\gg K$, thus requiring a complete forward \emph{and} backward pass over the new model. 

We propose \emph{Gradient-free Learning to Unify Experts} (\textbf{GLUE}), a method for learning expert mixing coefficients $\balpha$ via two-point simultaneous perturbation stochastic approximation (SPSA)~\cite{sadegh2002optimal, gu2021black}.
Starting from \eqref{eq:blending}, GLUE optimizes only the low-dimensional vector $\balpha$~\cite{jacobs1991adaptive, liudarts}.
To do so, at each iteration, it samples a random direction $\u \in \R^K$, evaluates the loss at $\balpha\pm\delta \u$ (two forward passes with blended weights $\btheta(\balpha \pm \delta \u)$), and updates $\balpha$ using the finite-difference estimate. No backpropagation is required---only forward passes---making GLUE lightweight and easy to deploy. While SPSA has been used in other learning tasks, it may be inaccurate in high dimensions~\cite{sadegh2002optimal, nesterov2017random, gu2021black}. The key insight behind GLUE is that the mixture coefficients are only $K$-dimensional, mitigating this common drawback of SPSA. Our contributions are threefold: \\
$\bullet\,$\textbf{Cost-effective prior construction}: We introduce GLUE, an expert parameter mixing method optimized via two-point SPSA, requiring only forward passes of the blended model. \\
$\bullet\,$\textbf{Theory for efficiency and stability}: 
We derive a per-iteration cost advantage over full-gradient mixing and show how the two-point estimator's variance (and thus sensitivity to perturbation radius/step size) scales with the mixture dimension, explaining the stability of low-dimensional updates in GLUE.
\\
$\bullet\,$\textbf{Empirical validation}: We demonstrate consistent gains of GLUE over existing weighting approaches and competitive performance with full-gradient $\balpha$ optimization, while being cheaper per iteration; the resulting single prior finetunes effectively on target domains.

In the following, we review related work (Section~\ref{sec:related}); introduce \textbf{GLUE} and its optimization procedure (Section~\ref{sec:method}); develop computational and statistical analyses (Section~\ref{sec:analysis}); present experiments against baselines (Section~\ref{sec:experiments}); discuss practical considerations and limitations (Section~\ref{sec:discussion}); and conclude (Section~\ref{sec:conclusion}).

%% file: sections/relatedworks.tex
\section{Related Work} \label{sec:related}
A straightforward way to combine experts is to weight each model in proportion to the amount of data used to train it---an idea resembling that of federated averaging (FedAvg), where model updates from multiple clients are aggregated based on clients' local dataset sizes in order to produce a unified model that (on average) performs well on all clients' data~\cite{mcmahan2017communication}. While simple and scalable, this strategy is agnostic to the \emph{target} domain. Under distribution changes, such as domain expansion, data-size weighting cannot adapt to target loss signals and can overemphasize large but misaligned sources.

Another common approach involves choosing or reweighting experts using a proxy metric (e.g., held-out accuracy) and then averaging the expert model parameters. Examples include ``model soups''---uniform averaging or greedy selection before averaging---and Fisher-weighted merging, which uses per-parameter Fisher information with model-level mixing as a hyperparameter ~\cite{wortsman2022model, matena2022merging, park2025fedbaf}. These methods can be effective when a reliable proxy set exists, but they decouple weight choice from the training objective and often require manual tuning or discrete selection steps.

A third line of work learns the expert mixture coefficients by minimizing a target-driven objective, ranging from scalar task-vector combinations (TIES-Merging~\cite{yadav2023ties}) to interference-aware merging and adaptive (un)supervised schemes that tune model-wise coefficients (AdaMerging)~\cite{ilharcoediting,yangadamerging}. These gradient-based methods require computing $\nabla_{\btheta(\balpha)}\mathcal{L}$ via backpropagation through the blended network and then taking inner products with expert parameters (e.g., $\partial\mathcal{L}/\partial \balpha_i=\langle \nabla_{\btheta(\balpha)}\mathcal{L},\,\btheta_i\rangle$), which is computationally expensive.

Zeroth-order methods replace backpropagation with finite differences. To compute the two-point simultaneous perturbation stochastic approximation (SPSA) estimator~\cite{spall2002multivariate}, $\balpha$ can be perturbed along a random direction to estimate a directional derivative with two forward passes, staying in the $K$-dimensional mixture space (where $K$ is the number of experts). However, the estimator's variance grows with dimension and curvature, demanding many directions or delicate tuning and often causing slow or unstable convergence for large $K$ or ill-conditioned losses~\cite{ZHANG2022110006}.

% Unlike data-size weighting, GLUE learns $\alpha$ directly on target domain data; unlike proxy-metric selection, we optimize a softmax-constrained mixture on the true training loss; and unlike gradient-based mixing, we use zeroth-order two-point SPSA---\emph{no backpropagation}. The method stays in the $K$-dimensional mixture space, needs only two forward passes per step, and yields a single deployable prior.

%% file: sections/methodology.tex
\section{Methodology} \label{sec:method}
In this section, we introduce \textbf{GLUE}, \emph{Gradient-free Learning to Unify Experts}, which builds a single prior by mixing fixed experts and learning the mixing parameters $\balpha$ via two-point SPSA.

\subsection{Objective}
Given $K$ pre-trained expert models with parameters $\{\btheta_i\}_{i=1}^K$, we aim to synthesize a \emph{single} pretrained model $\btheta(\balpha)$ (defined in~\eqref{eq:blending}) that serves as a strong prior for a \emph{new target domain} that is related to the sources but exhibits a shifted data distribution. Throughout, we \emph{only} optimize the low-dimensional mixing coefficients $\balpha$; the expert parameters $\{\btheta_i\}_{i=1}^K$ remain fixed.

Let $f(\cdot;{\btheta})$ denote the model (neural network) instantiated with parameters $\btheta$. For a minibatch $\mathcal{B}=\{(\x_j, \y_j)\}_{j=1}^B$ from the target domain and a standard supervised loss $\ell$ (cross-entropy for classification or squared error for regression), we minimize empirical risk:
\begin{align*}
\mathcal{L}(\balpha)=\frac{1}{B}\sum_{j=1}^B \ell\left(f(\x_j;{\btheta(\balpha)}),\,\y_j\right)
\end{align*}
We treat ${\balpha}$ as the sole optimization variable, and employ a zeroth-order (two-point) update on $\balpha$.

% and do not differentiate through $f$ or the $\btheta_i$; instead, we employ a zeroth-order (two-point) update on $\balpha$.

\subsection{Two-point SPSA updates}
Since our optimization variable is the $K$-dimensional mixture, and zeroth-order estimates are accurate in lower dimensions~\cite{sadegh2002optimal, nesterov2017random, gu2021black}, we estimate the gradient of $\mathcal{L}$ with respect to $\balpha$ by using a two-point simultaneous perturbation stochastic approximator (SPSA).

At each iteration, we draw a random direction vector $\u\in\mathbb{R}^K$ (e.g., by drawing entries of $\u$ i.i.d. from a standard Gaussian distribution and normalizing) and choose a small radius $\mu>0$. We form perturbed unconstrained coefficients $\balpha_\pm=\balpha\pm\mu \u$ and corresponding blended parameters $\btheta(\balpha_\pm)=\sum_{i=1}^K {\alpha}_{{\pm}_i}\,\btheta_i$. At both blends, we then evaluate the loss on the same minibatch of target-domain examples, keeping inference deterministic across the two probes (e.g., BatchNorm in evaluation mode and fixed randomness) so that differences reflect only the perturbation:
\begin{align}
\label{eq:twoevals}
L_\pm=\frac{1}{b}\sum_{j=1}^b \ell\left(f(\x_j;\btheta(\balpha_\pm)),\,\y_j\right)
\end{align}
The directional finite difference along $\u$ is
\begin{align*}
d(\balpha; \u)=\frac{L_+ - L_-}{2\mu}
\end{align*}
which induces the Gaussian-smoothing gradient estimate
\begin{align*}
\widehat{\nabla}_{\balpha} \mathcal{L}(\balpha)=d(\balpha;\u)\,\u
\end{align*}
If desired, we can reduce variance by additionally averaging over $m$ independently sampled directions $\{\u_r\}_{r=1}^m$ in the same iteration,
\begin{align}
\widehat{\nabla}_{\balpha} \mathcal{L}(\balpha)=\frac{1}{m}\sum_{r=1}^{m}\frac{L^{(r)}_+-L^{(r)}_-}{2\mu}\,\u_r
\label{eq:estimated graident}
\end{align}
with $\u_r$ sampled as above and $L^{(r)}_\pm$ computed as in~\eqref{eq:twoevals}. Finally, we update the mixture parameters:
\begin{align}
\label{eq:update}
\balpha \leftarrow \balpha - \eta \widehat{\nabla}_{\balpha} \mathcal{L}(\balpha)
\end{align}
where $\eta>0$ is the step size. 
In our experiments, we perform SPSA and momentum updates in an unconstrained $\bbeta$-space and map to $\balpha=\mathrm{softmax}(\bbeta)$, replacing~\eqref{eq:update}. Optimizing with respect to $\bbeta$ ensures that $\btheta(\balpha)$ remains within the convex hull of expert parameters.

% \begin{algorithm}[H]
% \footnotesize
% \caption{GLUE Training Procedure}
% \begin{algorithmic}[1]
% \State \textbf{Input:} mixture parameters $\balpha$, expert parameters $\{\btheta_i\}_{i=1}^K$, step size $\eta$, perturbation radius $\mu$, batch $\{(\x_j, \y_j)\}_{j=1}^b$, number of directions $m$
% \For{$r = 1$ to $m$}
%     \State Sample direction vector $\u_r$ (entries i.i.d. Gaussian or Rademacher)
%     \State Form perturbed coefficients: $\balpha_\pm = \balpha \pm \mu \u_r$
%     \State Compute blended parameters: $\theta(\balpha_\pm) = \sum_i \alpha_{\pm,i} \theta_i$
%     \State Evaluate losses: $L_\pm^{(r)} = \frac{1}{b}\sum_j \ell(f(\x_j; \theta(\balpha_\pm)), \y_j)$
%     \State Directional derivative: $d^{(r)} = \frac{L_+^{(r)} - L_-^{(r)}}{2\mu}$
% \EndFor
% \State Estimate gradient: $\widehat{\nabla}_{\balpha} \mathcal{L} = \frac{1}{m}\sum_{r=1}^m d^{(r)} \u_r$
% \State Update mixture parameters: $\mathbf{\beta} \gets \mathbf{\beta} - \eta\, \widehat{\nabla}_\mathbf{\beta} \mathcal{L}$, $\mathbf{\alpha} \gets \mathbf{\alpha}(\mathbf{\beta})$
% \EndFor
% \end{algorithmic}
% \end{algorithm}

\subsection{From mixture to prior}
We initialize $\balpha = \frac{1}{K} \mathbf{1}$ and iterate the two-point updates over minibatches until validation performance plateaus. Formally, we target the empirical-risk minimizer and map it to one blended model:
\begin{align*}
\label{eq:prior}
\balpha^\star \in \operatorname*{argmin}_{\balpha\in\mathbb{R}^K} \mathcal{L}(\balpha),\;
\quad
\btheta^\star=\btheta(\balpha^\star)
\end{align*}
The resulting $\btheta^\star$ serves as the starting point for standard gradient-based adaptation on the target domain, providing a strong prior that concentrates the experts' most relevant inductive biases.

%% file: sections/theoreticalanalysis.tex
\section{Theoretical Analysis} \label{sec:analysis}
In this section, we show that the two-point (SPSA) procedure used in \textbf{GLUE} is (i) computationally cheaper per iteration than gradient-based mixing and (ii) statistically stable when the mixture dimension \(K\) (the number of experts) is sufficiently small.
To do so, we leverage results in stochastic approximation and zeroth-order optimization~\cite{sadegh2002optimal, nesterov2017random, gu2021black}; under standard smoothness and bounded-noise assumptions, two-point SPSA converges to first-order stationary points, so GLUE inherits these classical guarantees~\cite{ghadimi2013stochastic}.

\subsection{Cost for updating \texorpdfstring{${\balpha}$}{}: two-point (SPSA) vs. full-gradient}
We compare a single \emph{SPSA step on ${\balpha}$} with a \emph{full-gradient step on ${\balpha}$} when the model parameters are blended as
${\btheta}({\balpha})=\sum_{i=1}^K \alpha_i {\btheta}_i$ and the experts $\{{\btheta}_i\}_{i=1}^K$ are fixed.
\\ \indent 
Let $F$ denote the cost of one forward pass of the blended network on a minibatch of data, and let the cost of a full backward pass be $\gamma F$ with $\gamma \ge 2$ (typical for deep neural networks~\cite{gruslys2016memory}). Let $C_{\text{mix}}$ be the cost of the multiply-accumulate operations used to form a blended parameter vector ${\btheta}({\balpha})$ from the $K$ experts: $C_{\text{mix}}=O(PK)$ memory operations for parameter dimension $P$. Let $D_{{\balpha}}$ be the cost of computing the gradient $\nabla_{\balpha} \mathcal{L}
=\big[\langle \nabla_{\btheta(\balpha)} \mathcal{L}, {\btheta}_1\rangle,\ldots,\langle \nabla_{\btheta(\balpha)} \mathcal{L}, {\btheta}_K\rangle\big]$ given $\nabla_{\btheta(\balpha)} \mathcal{L}$,
which is an $O(PK)$ dot-product accumulation.
Evaluating $\nabla_{\balpha} \mathcal{L}$ at the current mixture requires (i) computing ${\btheta}({\balpha})$ once, (ii) one forward pass and one backward pass to compute $\nabla_{\btheta(\balpha)} \mathcal{L}$, and (iii) $D_{\balpha}$ for computing $K$ inner products with $\langle \nabla_{\boldsymbol{\theta}(\balpha)} \mathcal{L}, \boldsymbol{\theta}_i\rangle$.
With one perturbation pair ($m = 1$), SPSA evaluates the loss at two nearby mixtures ${\balpha}_\pm={\balpha} \pm \delta\u$, requiring two blends and two forwards, but \emph{no} backward. Therefore:
\begin{align*}
T_{\text{full}}&=(1+\gamma) F + C_{\text{mix}} + D_{{\balpha}},\quad T_{\text{SPSA}}=2(F+C_{\text{mix}})
\end{align*}
Whenever $T_{\text{SPSA}} < T_{\text{full}}$, SPSA is cheaper than the full-gradient step 
% \begin{align}
% \frac{(1+\gamma) F + C_{\text{mix}} + D_{{\balpha}}}{2 (F+C_{\text{mix}})}\approx
% \frac{3F + C_{\text{mix}} + D_{\balpha}}{2(F+C_{\text{mix}})}
% \end{align}
(i.e., $
C_{\text{mix}} < (\gamma-1) F + D_{\balpha}
$). In practice, the condition is easily met since (i) the inner product cost $D_{\balpha}$ is typically of the same order as the mixing cost $C_{\text{mix}}$ and (ii) the cost of the backward pass is typically more than $2F$ ($\gamma\geq 2$).

\subsection{Variance of two-point estimates}
We analyze the variance of the two-point (SPSA) gradient estimator $\overline{\g}(\balpha)$, an average over $m$ i.i.d. random direction vectors $\{\u_r\}_{r=1}^m$:
\begin{equation}
\begin{aligned}
    \overline{\g}(\balpha) &= \frac{1}{m} \sum_{r=1}^m \widehat{\g}(\balpha; \u_r) \\
    \widehat{\g}(\balpha;\mathbf{u}) &= \left(\frac{{\mathcal{L}}(\balpha + \mu \mathbf{u}) - {\mathcal{L}}(\balpha - \mu \mathbf{u})}{2\mu}\right) \mathbf{u},\;\;\mu > 0 \label{eq:ghat}
\end{aligned}
\end{equation}
Let $\g = \nabla_{\balpha}\mathcal{L}$ denote the true gradient of the loss with respect to $\balpha$. As $\mu\rightarrow 0$, $\widehat{\g}$ becomes the directional derivative of $\mathcal{L}$ with respect to $\mathbf{u}$ and the approximation error scales with the square of $\mu$: $\widehat{\g} = (\g^\top \mathbf{u})\mathbf{u} + O(\mu^2)$. Since the directional derivative is an unbiased estimator of $g$, both $\widehat{\g}$  and $\overline{\g}$ are also unbiased~\cite{s2013stochastic}. Therefore, we bound the variance of $\overline{\g}$ in the following proposition---where the variance refers to the trace of the covariance matrix of the averaged gradient estimator $\overline{\g}$. 
\begin{proposition}[Variance Bound]
\label{prop:variance}
The variance of the two-point (SPSA) gradient estimator $\overline{\g}(\balpha)$ is upper bounded as:
\begin{align*}
    \E\left[\|\overline{\g} - \g\|^2\right] &\leq \frac{K-1}{mK} \sigma_{\mathrm{max}}(\bTheta)^2\|\nabla_{\btheta} \mathcal{L}\|^2 + O\left(\mu^2\right)
\end{align*}
where $\bTheta = [\btheta_1,\btheta_2,\dots,\btheta_K]$ and $\btheta=\bTheta\balpha$.
\end{proposition}
\begin{proof}
    We consider the first order Taylor approximation of~\eqref{eq:ghat}:
    \begin{align*}
        \widehat{\g}(\balpha; \u) &= (\g^\top \u)\u + O(\mu^2)
    \end{align*}
    We proceed to compute $\E\left[\|\widehat{\g} - \g\|^2\right]$ by using the fact that $\u$ being a unit isotropic random vector implies $\|\u\|=1$ and $\E[\u\u^\top] = \frac{1}{K} \I$.
    \begin{align*}
         \E&\left[\|\widehat{\g} - \g\|^2\right] =  \E\left[\left\|(\g^\top \u)\u + O(\mu^2) - \g\right\|^2\right]\\
         &=  \E\left[\left\|(\g^\top \u)\u\right\|^2 - 2(\g^\top \u)^2\right] + \|\g\|^2 + O(\mu^2)\\
         &=  -\g^\top\E\left[\u\u^\top \right] \g + \|\g\|^2 + O(\mu^2)=  \frac{K-1}{K}\|\g\|^2 + O(\mu^2)
    \end{align*}
    The variance with respect to $\overline{\g}$ is then:
    \begin{align}
        \mathrm{Var}[\overline{\g}] &= \frac{1}{m}\mathrm{Var}[\widehat{\g}] = \frac{K-1}{mK}\|\g\|^2 + O(\mu^2) \label{eq:var_gbound}
    \end{align}
    We proceed to bound $\|\g\|^2$ as:
    \begin{align*}
        \|\g\|^2 = \|\nabla_{\balpha} \mathcal{L}\|^2
        =\|\bTheta \nabla_{\btheta(\balpha)} \mathcal{L}\|^2
        \leq\sigma_{\mathrm{\max}}(\bTheta)^2 \|\nabla_{\btheta(\balpha)} \mathcal{L}\|^2
    \end{align*}
    Substituting the inequality above into~\eqref{eq:var_gbound} completes the proof. 
\end{proof}
The proposition demonstrates that 
% the variance of the gradient estimator scales as $\sigma_{\mathrm{\max}}(\bTheta)^2\|\nabla_{\btheta}\mathcal{L}(\btheta)\|^2$ with the number of experts $K$. Hence 
if $\sigma_{\mathrm{\max}}(\bTheta)$ is small, e.g., the expert parameters are more aligned, the variance grows slower in the number of experts $K$. The variance can be reduced with additional random direction samples (increasing $m$). Interestingly, we find that $m=1$ already provides stable updates for small $K$ empirically (see Section~\ref{sec:experiments}), which is consistent with the $1/m$ scaling.

Additionally, assuming that the loss function $\mathcal{L}$ is Lipschitz smooth with respect to $\theta$ leads to an upper bound on $\|\nabla_{\btheta}\mathcal{L}(\btheta)\|$.

%% file: sections/experimentalevaluation.tex
\begin{figure*}[ht]
  \centering
  \includegraphics[width=\linewidth]{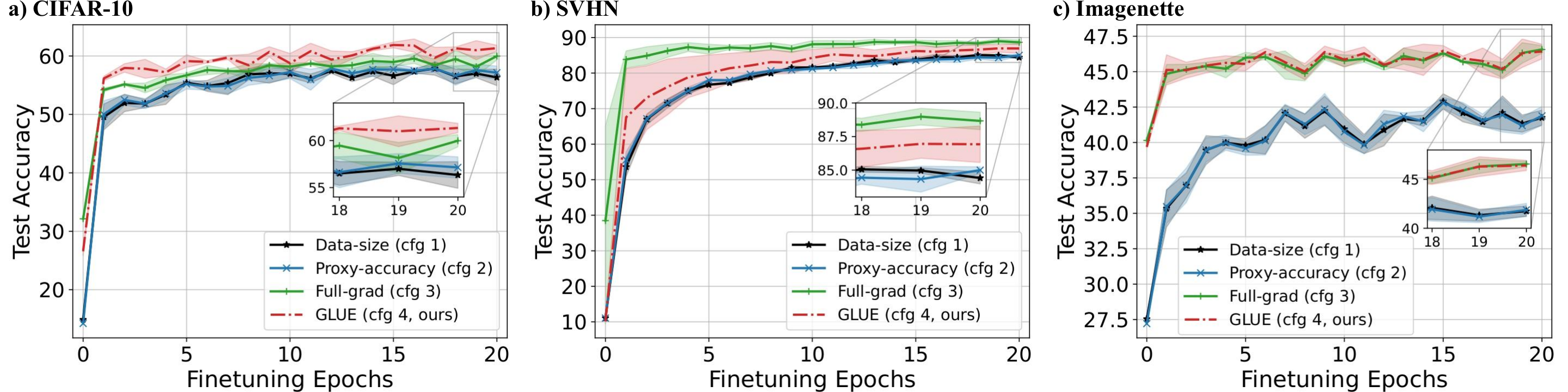}
  \vspace{-0.5 cm}
  \caption{Fine-tuning test accuracy vs.\ lightweight finetuning epochs (with fixed $\balpha$). Datasets: \textbf{(a)} CIFAR-10, \textbf{(b)} SVHN, \textbf{(c)} Imagenette. We compare Configs 1-4 (see Sec.~\ref{sec:experiments}). GLUE (Config 4; ours) shows similar trends to Config 3 (full backpropagation for gradient computation) and consistently achieves higher test accuracy and faster convergence than Configs 1-2 (heuristics for $\balpha$ based on data size and proxy accuracy).}
  \label{fig:experiments}
  \vspace{-0.5 cm}
\end{figure*}
\section{Experimental Evaluations}\label{sec:experiments}

We evaluate \textsc{GLUE} on CIFAR-10, SVHN, and Imagenette datasets using three backbones: ResNet-20, MobileNetV2, and an 8-layer ViT (patch size 8), respectively. We train $K=10$ expert models on heterogeneous splits of the base dataset. We adopt a non-IID experimental evaluation setup commonly used for federated learning: we sample per-class proportions for each expert from a Dirichlet distribution with concentration parameter $\delta = 0.5$, then subsample a total of 2000 images to match those class proportions. In general, $\delta \leq 0.5$ induces skewed (non-IID) splits; larger $\delta$ approaches IID. Each expert is trained for 40 epochs with a batch size of 64 and an Adam optimizer with a learning rate of 0.001 and momentum parameters $(0.9, 0.999)$. After training the experts, we evaluate four methods for determining $\balpha$, the mixture coefficients used to combine the expert models' parameters: \\
$\bullet\;$\textbf{Config 1 (Data-size):} $\alpha_i \propto n_i$, expert $i$'s pretraining dataset size. \\
$\bullet\;$\textbf{Config 2 (Proxy-accuracy):} $\alpha_i \propto \mathrm{Acc}_i$, the accuracy of expert $i$ on a small held-out validation set from the target domain. \\
$\bullet\;$\textbf{Config 3 (Full-grad $\balpha$):} Learn $\balpha$ by minimizing the target training loss with full backpropagation through $\btheta(\balpha)$; we parameterize $\balpha$ and use the same mini-batch schedule as \textsc{GLUE}. \\
$\bullet\;$\textbf{Config 4 (\textsc{GLUE}, ours):} Learn $\balpha$ using two-point SPSA on $\balpha$, requiring only forward evaluations (Sec.~\ref{sec:method}). \\
Configs~1 and 2 do not optimize on the target domain; Configs~3 and 4 \emph{learn} $\balpha$ using target-domain training data. In Configs~3 and 4, we use Adam optimizer with learning rate 1e-2, momentum parameters (0.9, 0.99) to update $\balpha$. Throughout training, $\btheta(\balpha)$ is constrained to lie in the convex hull of experts (i.e., $\sum_i\alpha_i=1$). Additionally for GLUE, we update $\balpha$ by randomly sampling only \emph{one} random direction $u$. 
We use 10,000 images sampled IID from the base dataset to learn the unification parameters $\balpha$ and finetune the models as described below. All unified models are subsequently evaluated on the original testing split of each of the base datasets. 

After determining $\balpha$, we form the unified initialization $\btheta(\balpha)$ for the target model and report:
(i) \emph{prior (zero-shot)} performance of $\btheta(\balpha)$ without any target training, and 
(ii) \emph{finetuned} performance after updating all network parameters on the target training set while \emph{keeping $\balpha$ fixed}. We also report the test accuracy of the \emph{individual experts} before blending in Table~\ref{tab:experts} for reference. We ran all finetuning experiments with three different random seeds and report the averaged test accuracy.

\begin{table}[t]
\centering
\caption{Test accuracies (\%) of the 10 experts evaluated on the target training dataset. We bolded the highest/lowest accuracies per row.}
\label{tab:experts}
\vspace{0.2cm}
\includegraphics[width=\linewidth]{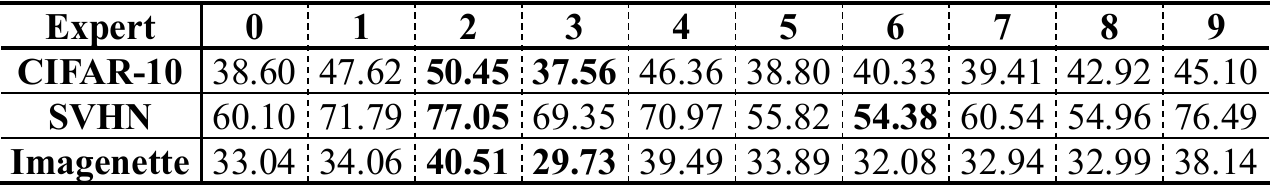}
\end{table}

Figure~\ref{fig:experiments} plots test accuracy during fine-tuning for the four configurations. Across all datasets and backbones, \textsc{GLUE} (Config~4) consistently outperforms data-size weighting (Config~1) by up to \textbf{6.7}\% on CIFAR-10, \textbf{3.8}\% on SVHN, and \textbf{8.5}\% on Imagenette. GLUE also outperforms proxy-accuracy weighting (Config~2) by up to \textbf{7.0}\% on CIFAR-10, \textbf{3.9}\% on SVHN, and \textbf{9.1}\% on Imagenette. \textsc{GLUE} attains similar performance to full-grad $\balpha$ (Config~3): GLUE outperforms full-grad on CIFAR-10 by up to \textbf{4.5}\% and has at most \textbf{1.4}\% and \textbf{0.5}\% variation in test accuracy on SVHN and Imagenette, respectively. To minimize computational cost in these experiments, GLUE computes only one two-point estimate to approximate the gradient at each iteration ($m=1$). However, the variation in test accuracy between GLUE and full-grad $\balpha$ (Config~3) can be reduced by averaging multiple two-point estimates to obtain better approximations to the true gradient---as in \eqref{eq:estimated graident}.
 
In summary, Figure~\ref{fig:experiments} shows that \textsc{GLUE} (Config~4) starts from a stronger prior and converges faster to a higher test accuracy than the alternatives, closely tracking the performance of the full-gradient method (Config~3).

These results highlight that \textbf{1)} the performance gains from learning $\balpha$ persist even after finetuning the target model and \textbf{2)} GLUE mimics the gains of full-gradient $\balpha$ learning while providing computational benefits by relying only on forward evaluations instead of costly backpropagation.  

%% file: sections/discussion.tex
\section{Discussion}\label{sec:discussion}
In \textbf{GLUE}, each step forms two blends and runs two forward passes; there is no backward pass. The per-step cost advantage over full-gradient mixing follows from
\begin{equation*}
\begin{aligned}
T_{\text{full}}-T_{\text{SPSA}}
&=(1+\gamma)F+C_{\text{mix}}+D_{\balpha}-2(F+C_{\text{mix}}) \\
&=(\gamma-1)F - C_{\text{mix}} + D_{\balpha}
\end{aligned}
\end{equation*}
As neural networks grow, the backward term $(\gamma-1)F$ dominates, so the gap widens in favor of SPSA. In practice, we also amortize $C_{\text{mix}}$ by caching the two blended parameter sets across the paired forward passes.

In this work, GLUE assumes that the architectures are compatible, meaning they share backbones and heads. When the heads differ, we mix the backbone and select or align a single head, optionally refreshing the BatchNorm statistics.

Other limitations include the convex-hull restriction (the target optimum may lie outside the experts' span), possible misalignment across experts (objectives or preprocessing), hyperparameter sensitivity of SPSA (too large radius $\mu$ of \eqref{eq:estimated graident} biases the estimate; too small increases noise), and variance growing with $K$ regardless of classical SPSA convergence-to-stationary-point guarantees. These can be mitigated by light calibration, averaging a few directions, or modest schedules, but remain trade-offs.

%% file: sections/conclusion.tex
\section{Conclusion} \label{sec:conclusion}
We presented \textbf{GLUE}, a forward-only approach that learns a single deployable prior by mixing pretrained experts in a convex manner and optimizing the mixture via two-point SPSA. The method avoids backpropagation and is accompanied by cost and variance analyses that explain its efficiency and stability. 
Empirically, GLUE delivers consistent improvements over data-size weighting and proxy-metric baselines while remaining lightweight for deployment when compared to the learned full gradient mixing baseline.

%% file: mybib.bib
@inproceedings{mcmahan2017communication,
  title={Communication-efficient learning of deep networks from decentralized data},
  author={McMahan, Brendan and Moore, Eider and Ramage, Daniel and Hampson, Seth and y Arcas, Blaise Aguera},
  booktitle={International Conference on Artificial Intelligence and Statistics},
  pages={1273--1282},
  year={2017},
  organization={PMLR}
}

@inproceedings{wortsman2022model,
  title={Model soups: averaging weights of multiple fine-tuned models improves accuracy without increasing inference time},
  author={Wortsman, Mitchell and Ilharco, Gabriel and Gadre, Samir Ya and Roelofs, Rebecca and Gontijo-Lopes, Raphael and Morcos, Ari S and Namkoong, Hongseok and Farhadi, Ali and Carmon, Yair and Kornblith, Simon and others},
  booktitle={International Conference on Machine Learning},
  pages={23965--23998},
  year={2022},
  organization={PMLR}
}

@article{hinton2015distilling,
  title={Distilling the knowledge in a neural network},
  author={Hinton, Geoffrey and Vinyals, Oriol and Dean, Jeff},
  journal={arXiv preprint arXiv:1503.02531},
  year={2015}
}

@inproceedings{chaudhari2023knowledge,
  title={{Knowledge Distillation by Compressive Sampling}},
  author={Chaudhari, Shreyas and Moura, Jos{\'e} M. F.},
  booktitle={2023 57th Asilomar Conference on Signals, Systems, and Computers},
  pages={882--886},
  year={2023},
  organization={IEEE}
}

@article{matena2022merging,
  title={Merging models with fisher-weighted averaging},
  author={Matena, Michael S and Raffel, Colin A},
  journal={Advances in Neural Information Processing Systems},
  volume={35},
  pages={17703--17716},
  year={2022}
}

@inproceedings{ilharcoediting,
  title={Editing models with task arithmetic},
  author={Ilharco, Gabriel and Ribeiro, Marco Tulio and Wortsman, Mitchell and Schmidt, Ludwig and Hajishirzi, Hannaneh and Farhadi, Ali},
  booktitle={The Eleventh International Conference on Learning Representations},
  year={2023}
}

@article{yadav2023ties,
  title={Ties-merging: Resolving interference when merging models},
  author={Yadav, Prateek and Tam, Derek and Choshen, Leshem and Raffel, Colin A and Bansal, Mohit},
  journal={Advances in Neural Information Processing Systems},
  volume={36},
  pages={7093--7115},
  year={2023}
}

@inproceedings{yangadamerging,
  title={AdaMerging: Adaptive Model Merging for Multi-Task Learning},
  author={Yang, Enneng and Wang, Zhenyi and Shen, Li and Liu, Shiwei and Guo, Guibing and Wang, Xingwei and Tao, Dacheng},
  booktitle={The Twelfth International Conference on Learning Representations},
  year={2024}
}

@inproceedings{noroozi2018boosting,
  title={Boosting self-supervised learning via knowledge transfer},
  author={Noroozi, Mehdi and Vinjimoor, Ananth and Favaro, Paolo and Pirsiavash, Hamed},
  booktitle={Proceedings of the IEEE Conference on Computer Vision and Pattern Recognition},
  pages={9359--9367},
  year={2018}
}

@inproceedings{you2021logme,
  title={Logme: Practical assessment of pre-trained models for transfer learning},
  author={You, Kaichao and Liu, Yong and Wang, Jianmin and Long, Mingsheng},
  booktitle={International Conference on Machine Learning},
  pages={12133--12143},
  year={2021},
  organization={PMLR}
}

@inproceedings{park2025fedbaf,
  title={{FedBaF: Federated Learning Aggregation Biased by a Foundation Model}},
  author={Park, Jong-Ik and Pranav, Srinivasa and Moura, Jos\'e M.F. and Joe-Wong, Carlee},
  booktitle={International Conference on Artificial Intelligence and Statistics},
  pages={676--684},
  year={2025},
  organization={PMLR}
}

@inproceedings{turan2020achieving,
  title={{Achieving multi-accent ASR via unsupervised acoustic model adaptation}},
  author={Turan, Mehmet Ali Tugtekin and Vincent, Emmanuel and Jouvet, Denis},
  booktitle={INTERSPEECH 2020},
  year={2020}
}

@article{akbarian2023evaluating,
  title={Evaluating knowledge transfer in the neural network for medical images},
  author={Akbarian, Sina and Seyyed-Kalantari, Laleh and Khalvati, Farzad and Dolatabadi, Elham},
  journal={IEEE Access},
  volume={11},
  pages={85812--85821},
  year={2023},
  publisher={IEEE}
}

@article{pal2024transfer,
  title={{Transfer Learning in Weather Prediction: Why, How, and What Should}},
  author={Pal, Sankar Kumar and Dutta, Debashree},
  journal={Journal of Computational and Cognitive Engineering},
  volume={3},
  number={4},
  pages={324--347},
  year={2024}
}

@article{shu2022hub,
  title={Hub-pathway: transfer learning from a hub of pre-trained models},
  author={Shu, Yang and Cao, Zhangjie and Zhang, Ziyang and Wang, Jianmin and Long, Mingsheng},
  journal={Advances in Neural Information Processing Systems},
  volume={35},
  pages={32913--32927},
  year={2022}
}

@inproceedings{fu2024exploring,
  title={Exploring adapter-based transfer learning for recommender systems: Empirical studies and practical insights},
  author={Fu, Junchen and Yuan, Fajie and Song, Yu and Yuan, Zheng and Cheng, Mingyue and Cheng, Shenghui and Zhang, Jiaqi and Wang, Jie and Pan, Yunzhu},
  booktitle={Proceedings of the 17th ACM International Conference on Web Search and Data Mining},
  pages={208--217},
  year={2024}
}

@article{zhou2024comprehensive,
  title={A comprehensive survey on pretrained foundation models: A history from bert to chatgpt},
  author={Zhou, Ce and Li, Qian and Li, Chen and Yu, Jun and Liu, Yixin and Wang, Guangjing and Zhang, Kai and Ji, Cheng and Yan, Qiben and He, Lifang and others},
  journal={International Journal of Machine Learning and Cybernetics},
  pages={1--65},
  year={2024},
  publisher={Springer}
}

@inproceedings{yuan2023power,
  title={On the power of foundation models},
  author={Yuan, Yang},
  booktitle={International Conference on Machine Learning},
  pages={40519--40530},
  year={2023},
  organization={PMLR}
}

@article{moslemi2024survey,
  title={A survey on knowledge distillation: Recent advancements},
  author={Moslemi, Amir and Briskina, Anna and Dang, Zubeka and Li, Jason},
  journal={Machine Learning with Applications},
  volume={18},
  pages={100605},
  year={2024},
  publisher={Elsevier}
}

@inproceedings{mirzadeh2020improved,
  title={Improved knowledge distillation via teacher assistant},
  author={Mirzadeh, Seyed Iman and Farajtabar, Mehrdad and Li, Ang and Levine, Nir and Matsukawa, Akihiro and Ghasemzadeh, Hassan},
  booktitle={Proceedings of the AAAI Conference on Artificial Intelligence},
  volume={34},
  pages={5191--5198},
  year={2020}
}

@article{zagoruyko2016paying,
  title={Paying more attention to attention: Improving the performance of convolutional neural networks via attention transfer},
  author={Zagoruyko, Sergey and Komodakis, Nikos},
  journal={arXiv preprint arXiv:1612.03928},
  year={2016}
}

@article{micaelli2019zero,
  title={Zero-shot knowledge transfer via adversarial belief matching},
  author={Micaelli, Paul and Storkey, Amos J.},
  journal={Advances in Neural Information Processing Systems},
  volume={32},
  year={2019}
}

@article{cohenask,
  title={Ask Your Distribution Shift if Pre-Training is Right for You},
  author={Cohen-Wang, Benjamin and Vendrow, Joshua and Madry, Aleksander},
  journal={Transactions on Machine Learning Research},
  year={2024}
}

@inproceedings{shi2022robust,
  title={How robust are pre-trained models to distribution shift?},
  author={Shi, Yuge and Daunhawer, Imant and Vogt, Julia E and Torr, Philip and Sanyal, Amartya},
  booktitle={ICML 2022: Workshop on Spurious Correlations, Invariance and Stability},
  year={2022}
}

@inproceedings{yuan2021reinforced,
  title={Reinforced multi-teacher selection for knowledge distillation},
  author={Yuan, Fei and Shou, Linjun and Pei, Jian and Lin, Wutao and Gong, Ming and Fu, Yan and Jiang, Daxin},
  booktitle={Proceedings of the AAAI Conference on Artificial Intelligence},
  volume={35},
  pages={14284--14291},
  year={2021}
}

@inproceedings{wu2021one,
  title={One Teacher is Enough? Pre-trained Language Model Distillation from Multiple Teachers},
  author={Wu, Chuhan and Wu, Fangzhao and Huang, Yongfeng},
  booktitle={Findings of the Association for Computational Linguistics: ACL-IJCNLP 2021},
  pages={4408--4413},
  year={2021}
}

@inproceedings{shi2025fedawa,
  title={FedAWA: Adaptive Optimization of Aggregation Weights in Federated Learning Using Client Vectors},
  author={Shi, Changlong and Zhao, He and Zhang, Bingjie and Zhou, Mingyuan and Guo, Dandan and Chang, Yi},
  booktitle={Proceedings of the IEEE Conference on Computer Vision and Pattern Recognition},
  pages={30651--30660},
  year={2025}
}

@inproceedings{boiarov2022simultaneous,
  title={Simultaneous Perturbation Method for Multi-task Weight Optimization in One-Shot Meta-learning},
  author={Boiarov, Andrei and Khabarlak, Kostiantyn and Yastrebov, Igor},
  booktitle={International Conference on Neural Information Processing},
  pages={373--384},
  year={2022},
  organization={Springer}
}

@inproceedings{xiang2020learning,
  title={Learning from multiple experts: Self-paced knowledge distillation for long-tailed classification},
  author={Xiang, Liuyu and Ding, Guiguang and Han, Jungong},
  booktitle={European Conference on Computer Vision},
  pages={247--263},
  year={2020},
  organization={Springer}
}

@article{mu2025comprehensive,
  title={A comprehensive survey of mixture-of-experts: Algorithms, theory, and applications},
  author={Mu, Siyuan and Lin, Sen},
  journal={arXiv preprint arXiv:2503.07137},
  year={2025}
}

@article{gu2021black,
  title={Black-box reductions for zeroth-order gradient algorithms to achieve lower query complexity},
  author={Gu, Bin and Wei, Xiyuan and Gao, Shangqian and Xiong, Ziran and Deng, Cheng and Huang, Heng},
  journal={Journal of Machine Learning Research},
  volume={22},
  number={170},
  pages={1--47},
  year={2021}
}

@article{sadegh2002optimal,
  title={Optimal random perturbations for stochastic approximation using a simultaneous perturbation gradient approximation},
  author={Sadegh, Payman and Spall, James C},
  journal={IEEE Transactions on Automatic Control},
  volume={43},
  number={10},
  pages={1480--1484},
  year={2002},
  publisher={IEEE}
}

@article{jacobs1991adaptive,
  title={Adaptive mixtures of local experts},
  author={Jacobs, Robert A and Jordan, Michael I and Nowlan, Steven J and Hinton, Geoffrey E},
  journal={Neural Computation},
  volume={3},
  number={1},
  pages={79--87},
  year={1991},
  publisher={MIT Press}
}

@inproceedings{liudarts,
  title={DARTS: Differentiable Architecture Search},
  author={Liu, Hanxiao and Simonyan, Karen and Yang, Yiming},
  booktitle={International Conference on Learning Representations},
  year={2019}
}

@article{spall2002multivariate,
  title={Multivariate stochastic approximation using a simultaneous perturbation gradient approximation},
  author={Spall, James C},
  journal={IEEE Transactions on Automatic Control},
  volume={37},
  number={3},
  pages={332--341},
  year={2002},
  publisher={IEEE}
}

@article{ZHANG2022110006,
title = {A new one-point residual-feedback oracle for black-box learning and control},
journal = {Automatica},
volume = {136},
pages = {110006},
year = {2022},
issn = {0005-1098},
doi = {https://doi.org/10.1016/j.automatica.2021.110006},
url = {https://www.sciencedirect.com/science/article/pii/S000510982100532X},
author = {Yan Zhang and Yi Zhou and Kaiyi Ji and Michael M. Zavlanos},
}

@article{gruslys2016memory,
  title={Memory-efficient backpropagation through time},
  author={Gruslys, Audrunas and Munos, R{\'e}mi and Danihelka, Ivo and Lanctot, Marc and Graves, Alex},
  journal={Advances in Neural Information Processing Systems},
  volume={29},
  year={2016}
}

@INPROCEEDINGS{NTK_P2P_ICASSP,
  author={Chaudhari, Shreyas and Pranav, Srinivasa and Anand, Emile and Moura, José M. F.},
  booktitle={ICASSP 2025 - 2025 IEEE International Conference on Acoustics, Speech and Signal Processing (ICASSP)}, 
  title={{Peer-to-Peer Learning Dynamics of Wide Neural Networks}}, 
  year={2025},
  volume={},
  number={},
  pages={1-5},
  doi={10.1109/ICASSP49660.2025.10890126}}

@INPROCEEDINGS{5GIoT_P2P_Asilomar,
  author={Pranav, Srinivasa and Moura, José M. F.},
  booktitle={2024 58th Asilomar Conference on Signals, Systems, and Computers}, 
  title={{Peer-to-Peer Deep Learning for Beyond-5G IoT}}, 
  year={2024},
  volume={},
  number={},
  pages={1000-1004},
  doi={10.1109/IEEECONF60004.2024.10943006}}

@INPROCEEDINGS{NonIID_P2P_Asilomar,
  author={Pranav, Srinivasa and Moura, José M. F.},
  booktitle={2023 57th Asilomar Conference on Signals, Systems, and Computers}, 
  title={{Peer-to-Peer Learning+Consensus with Non-IID Data}}, 
  year={2023},
  volume={},
  number={},
  pages={709-713},
  doi={10.1109/IEEECONF59524.2023.10476902}}

@article{nesterov2017random,
  title={Random gradient-free minimization of convex functions},
  author={Nesterov, Yurii and Spokoiny, Vladimir},
  journal={Foundations of Computational Mathematics},
  volume={17},
  number={2},
  pages={527--566},
  year={2017},
  publisher={Springer}
}

@article{ghadimi2013stochastic,
  title={Stochastic first-and zeroth-order methods for nonconvex stochastic programming},
  author={Ghadimi, Saeed and Lan, Guanghui},
  journal={SIAM Journal on Optimization},
  volume={23},
  number={4},
  pages={2341--2368},
  year={2013},
  publisher={SIAM}
}

@book{s2013stochastic,
  title={Stochastic Recursive Algorithms for Optimization: Simultaneous Perturbation Methods},
  author={S. Bhatnagar and Prasad, HL and Prashanth, LA},
  year={2013},
  publisher={Springer}
}
